\documentclass[10pt,twocolumn,letterpaper]{article}

\usepackage{cvpr}
\usepackage{url}
\usepackage{booktabs}
\usepackage{multirow}
\usepackage{makecell}
\usepackage{graphicx}
\usepackage{caption}
\usepackage{amsmath}
\usepackage{amssymb}
\usepackage{pifont}
\usepackage{listings}
\usepackage{xcolor}
\usepackage{amsthm}
\usepackage{enumitem}
\usepackage{colortbl}
\usepackage[symbol]{footmisc}
\newcommand{\customfootnotetext}[2]{%
  \begingroup
  \renewcommand{\thefootnote}{#1}%
  \footnotetext{#2}%
  \endgroup
}
\newcommand{\fullalgoname}{Lay-Your-Scene}
\newcommand{\algoname}{LayouSyn}

\definecolor{darkgreen}{RGB}{0,100,0}  

\definecolor{cvprblue}{rgb}{0.21,0.49,0.74}
\usepackage[pagebackref,breaklinks,colorlinks,allcolors=cvprblue]{hyperref}

\title{Lay-Your-Scene: Natural Scene Layout Generation with Diffusion Transformers}

\author{Divyansh Srivastava\textsuperscript{1} \quad
Xiang Zhang\textsuperscript{1}
\quad
He Wen\textsuperscript{2}\textsuperscript{*}
\quad
Chenru Wen\textsuperscript{2}\textsuperscript{*}
\quad
Zhuowen Tu\textsuperscript{1} \\
\textsuperscript{1}UC San Diego \quad 
\textsuperscript{2}Tsingua University \\
}

\begin{document}
\definecolor{Gray}{gray}{0.9}
\newcolumntype{a}{>{\columncolor{Gray}}c}

\twocolumn[{%
			\renewcommand\twocolumn[1][]{#1}%
			\maketitle
			\begin{center}
				\vspace{-5mm}
				\includegraphics[width=\linewidth,trim=1mm 2mm 1mm 0mm,clip]{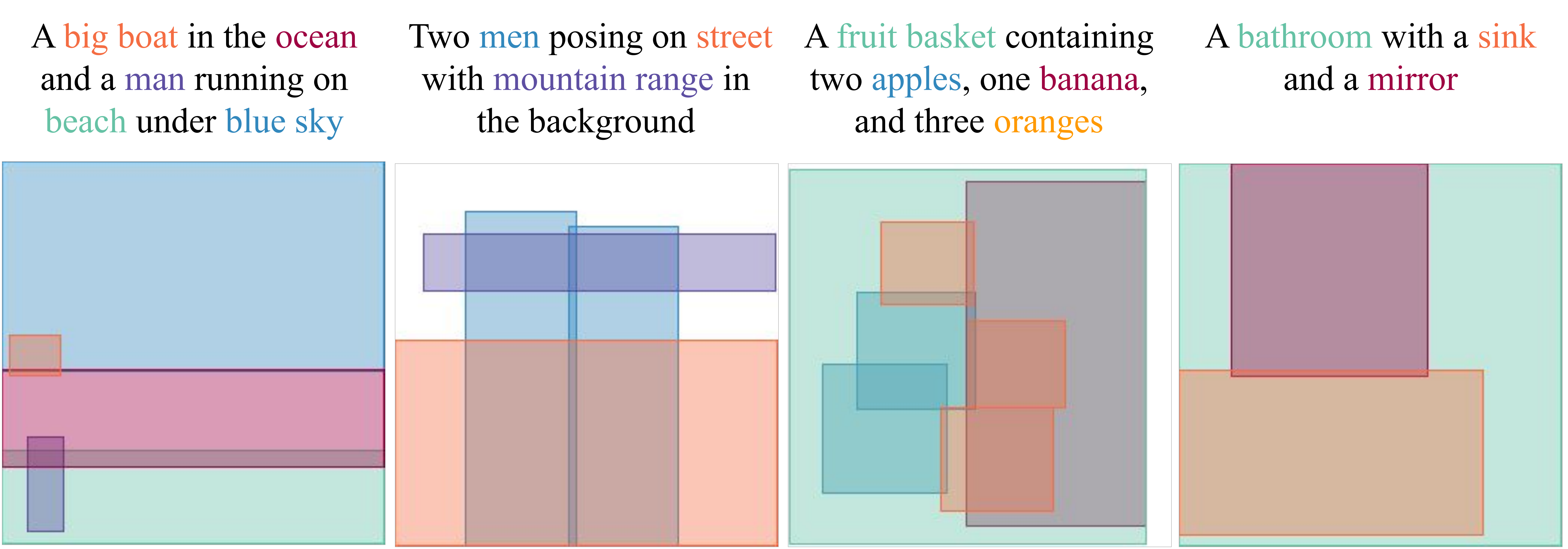}
				\captionof{figure}{{Text to natural scene layout generation} with \textbf{\algoname{}}. \algoname{} demonstrates superior scene awareness, generating layouts with high geometric plausibility and strictly adhering to numerical and spatial constraints. Object nouns in the prompts are highlighted with corresponding colors in the layout.}
                \vspace{0.5em}
				\label{fig:teaser}
			\end{center}
		}]

\customfootnotetext{*}{Work done during internship at UC San Diego}

\begin{abstract}    
    We present \textbf{\fullalgoname{}} (shorthand \algoname{}), a novel text-to-layout generation pipeline for natural scenes. Prior scene layout generation methods are either closed-vocabulary or use proprietary large language models for open-vocabulary generation, limiting their modeling capabilities and broader applicability in controllable image generation. In this work, we propose to use lightweight open-source language models to obtain scene elements from text prompts and a novel aspect-aware diffusion Transformer architecture trained in an open-vocabulary manner for conditional layout generation. Extensive experiments demonstrate that \algoname{} outperforms existing methods and achieves state-of-the-art performance on challenging spatial and numerical reasoning benchmarks. Additionally, we present two applications of \algoname{}. First, we show that coarse initialization from large language models can be seamlessly combined with our method to achieve better results. Second, we present a pipeline for adding objects to images, demonstrating the potential of \algoname{} in image editing applications.
\end{abstract}

\section{Introduction}
 
With the advancement of text-to-image generative models~\cite{ramesh2021zero,nichol2021glide,rombach2022high,chen2023pixart,xue2024raphael}, there has been a growing interest in controllable image generation~\cite{li2023gligenopensetgroundedtexttoimage,zhang2023adding}, where users can explicitly control the spatial locations~\cite{xie2023boxdiff,wang2024instancediffusion,li2023gligenopensetgroundedtexttoimage}, and counts~\cite{binyamin2024make,yang2023reco} of objects in the generated images. While existing frameworks ~\cite{xie2023boxdiff,wang2024instancediffusion,li2023gligenopensetgroundedtexttoimage} can achieve satisfactory control over image generation, they require users to provide fine-grained conditioning inputs, such as plausible scene layouts. Therefore, a text-to-layout generation framework is necessary to reduce the manual, time-consuming, and costly effort of obtaining such layouts.

Existing works have addressed layout generation with various architectures for modeling inter-element relationships within scene elements: variational autoencoders \cite{jyothi2019layoutvae}, GANs \cite{li2019layoutgan}, autoregressive transformers \cite{gupta2021layouttransformer}, and more recently, continuous \cite{wang2024dolfin} or discrete \cite{inoue2023layoutdm, zhang2023layoutdiffusion} diffusion processes. While these methods demonstrate competitive performance across various benchmarks, they primarily focus on unconditional layout generation, such as document layouts. Furthermore, these models often assume a fixed set of object categories and struggle with complex text conditions, restricting their application in controlled, end-to-end, open-vocabulary text-to-image generation pipelines.

Recently, a line of work \cite{feng2023layoutgpt, gani2023llm, feng2023ranni, lian2023llm} has leveraged generalization and in-context learning capabilities of proprietary Large Language Models (LLMs), such as GPT \cite{ouyang2022training}, for open-vocabulary layout generation. While these LLM-based approaches can generate reasonable scene layouts, they often produce unrealistic object aspect ratios or unnatural bounding box placements \cite{gani2023llm}. Additionally, reliance on proprietary closed-source LLMs introduces opacity, increased latency, and higher costs, limiting their broader applicability in areas including controllable image generation.

To address these limitations, we introduce \textbf{\fullalgoname{}} (\algoname{}), a novel text-to-natural scene layout generation pipeline that combines the text understanding capabilities of lightweight open-source language models with the strong generative capabilities of diffusion models in the visual domain. We frame the scene layout generation task as a two-stage process. First, a lightweight, open-source language model extracts a set of relevant object descriptions from the text prompt describing the scene. Second, we train a novel, aspect-ratio-aware diffusion model in an open-vocabulary manner for scene layout generation. Our contributions are summarized as follows:

\begin{enumerate}
    \item \textbf{Novel framework and architecture:} We present {\algoname{}}, a novel open-vocabulary text to natural scene layout generation pipeline. Our pipeline adopts a lightweight open-source language model to predict object descriptions from text prompts and a novel aspect-aware Diffusion-Transformer architecture with improved alignment between local and global conditions for layout generation. A schematic illustration of the training and inference pipeline is demonstrated in \cref{fig:pipeline}.
    \item \textbf{State-of-the-art results:} Extensive qualitative and quantitative experiments demonstrate that \algoname{} can generate semantically and geometrically plausible layouts, outperform existing methods on scene layout generation, and achieve state-of-the-art performance on challenging spatial and numerical reasoning tasks. 
    \item \textbf{Versatile applications:} We demonstrate the versatility of \algoname{} in two key applications: First, we show that coarse layout initialization from LLMs, such as GPT, can be seamlessly integrated with our method, achieving better results. Second, we utilize \algoname{} for layout completion and introduce an automated pipeline for adding objects to images, highlighting \algoname{}'s potential in image editing applications.
\end{enumerate}

\begin{figure*}[t]
	\centering
	\includegraphics[width=\linewidth]{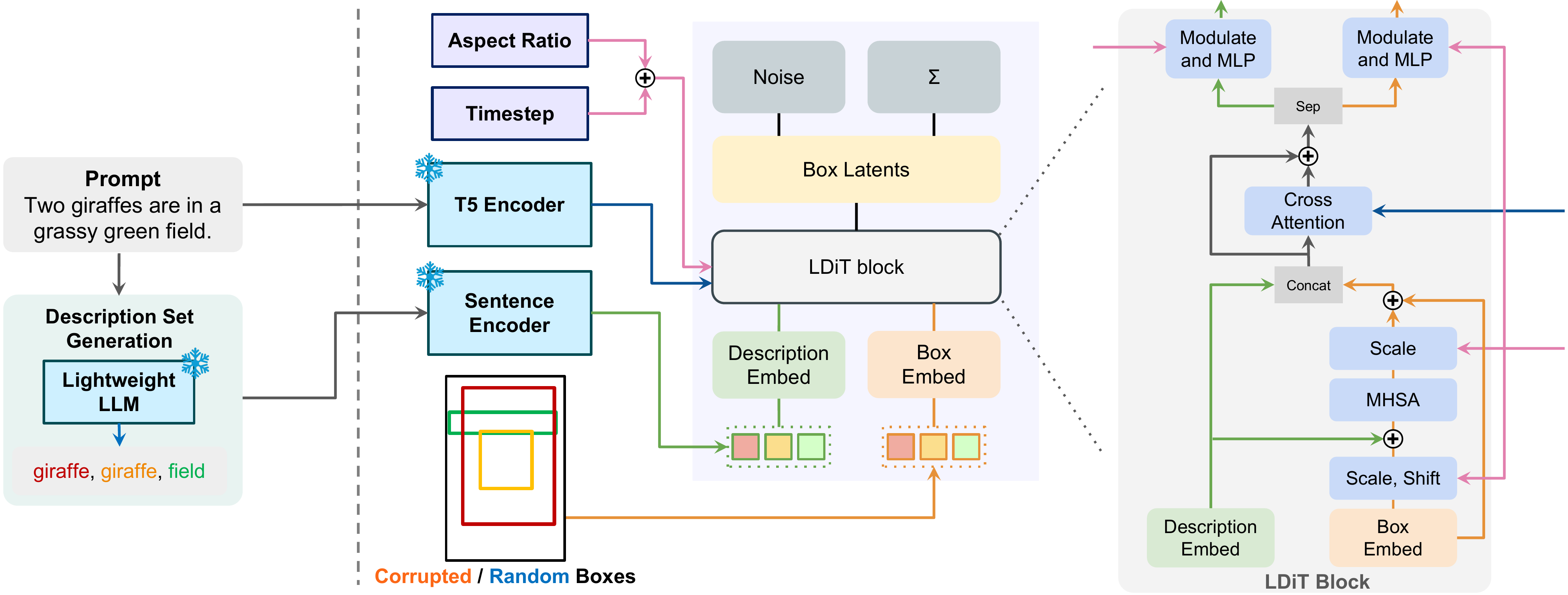}
	\caption{\textbf{Overview of inference pipeline for \algoname{}}. We frame the scene layout generation task as a two-stage process. First, a lightweight language model extracts a set of relevant object descriptions from the text prompt describing the scene. Second, a trained diffusion model generates layouts conditioned on the text prompt and object descriptions. Note that \textit{Concat} in the LDiT block refers to the concatenation of description and bounding box tokens along the token dimension, while \textit{Sep} denotes their separation.}
	\label{fig:pipeline}
\end{figure*}

\section{Related Work}

\paragraph{Document layout generation} Previous works on closed-vocabulary layout generation (with fixed object categories) have proposed sophisticated architectures to address the task. For instance, LayoutGAN~\cite{li2019layoutgan}, a GAN-based framework, generates both object labels and bounding boxes from noise simultaneously. However, it cannot perform layout generation conditioned on specific description sets, and its evaluations are limited to documents with few objects in the layout. LayoutVAE~\cite{jyothi2019layoutvae} improves upon this by generating layouts conditioned on description sets autoregressively using LSTM-based VAEs, handling a larger number of object categories, such as those found in the COCO dataset~\cite{lin2015microsoftcococommonobjects}. VTN~\cite{arroyo2021variational} brings further improvements by incorporating Transformers as a building block for VAEs, allowing better modeling of inter-element relationships within a layout. Another line of research formulates layout generation as a sequence generation problem: BLT~\cite{kong2022blt} employs a bidirectional Transformer for iterative decoding, while LayoutTransformers~\cite{gupta2021layouttransformer} uses the standard next-token prediction approach for generating layout elements. LayoutFormer++~\cite{jiang2023layoutformer++} introduces decoding space restrictions to effectively align layouts with user-defined constraints. More recently, diffusion models have been explored for layout generation: Dolfin~\cite{wang2024dolfin} and LayoutDM~\cite{chai2023layoutdm} apply continuous diffusion in the bounding box coordinate space, while LayoutDM~\cite{inoue2023layoutdm} and LayoutDiffusion~\cite{zhang2023layoutdiffusion} address the task using discrete diffusion on both coordinate and type tokens. These models have also shown utility in conditional generation tasks, such as layout refinement and type-conditioned generation. Despite these advances, most of these approaches are primarily designed and benchmarked on document layout generation tasks, and the closed-vocabulary nature limits their applicability and generalizability to natural scene layout generation.

\paragraph{Scene layout generation} Open-vocabulary layout generation is an essential component of controllable text-to-image generation pipelines. Methods such as GLIGEN \cite{li2023gligenopensetgroundedtexttoimage}, Instance Diffusion \cite{wang2024instancediffusioninstancelevelcontrolimage}, ReCo \cite{yang2023reco}, and Boxdiff \cite{xie2023boxdiff} leverage input scene layouts and fine-grained object-level natural language prompts for layout-to-image generation. Unfortunately, these approaches require human intervention to provide scene layouts, which is time-consuming and costly. This necessitates open-vocabulary layout generation frameworks that accommodate any object-level natural language prompts without being restricted to fixed categories. Recent approaches predominantly address this challenge by leveraging the reasoning capabilities of large language models (LLMs) like ChatGPT~\cite{ouyang2022training}. For instance, LayoutGPT \cite{feng2023layoutgpt} employs a style sheet-like scene structure combined with the in-context learning capabilities of LLMs to generate layouts using proprietary GPT models. Additionally, LayoutGPT \cite{feng2023layoutgpt} introduces the Numerical and Spatial Reasoning (NSR-1K) benchmark to evaluate the spatial and counting accuracy of generated layouts, which we also use to assess the performance of our \algoname{}. LLM Blueprint~\cite{gani2023llm} generates object descriptions alongside layouts to better guide the image generation pipeline. While these approaches achieve promising results, their reliance on proprietary LLMs for layout generation reduces transparency and incurs additional financial costs. Also, they tend to generate unrealistic layouts as demonstrated in \cite{gani2023llm}. In contrast, \algoname{} adopts lightweight open-source language models to predict objects from text prompts and a novel open-vocabulary diffusion-Transformer-based architecture trained in a scene-aware manner to generate layouts at any aspect ratio.

\paragraph{Diffusion Transformers} Diffusion Transformers were first introduced in \cite{peebles2023scalable} to tackle class-conditional image generation. The self-attention layers in Transformers allow for more effective modeling of relationships between tokens. Beyond text-to-image generation~\cite{esser2024scaling}, this architecture has been adapted for layout generation~\cite{inoue2023layoutdm,wang2024dolfin}, 3D shape generation~\cite{mo2023dit,xu2024bayesian}, and video generation~\cite{videoworldsimulators2024}. Our approach builds on diffusion Transformers for layout generation, with a diffusion process that operates directly in the continuous bounding box coordinate space, eliminating the need for VAE encoding.

\section{Method}

We present \textbf{\algoname{}}, a novel diffusion-based pipeline for open-vocabulary text-to-layout (T2L) generation. Our approach decomposes T2L generation into two sub-tasks: (1) generating object descriptions from an input text prompt and (2) generating layouts conditioned on object descriptions and the input prompt.  Formally, given a prompt $p$ describing a natural scene $\mathcal{S}$, we aim to generate a layout $\mathcal{L}$, represented by a set of tuples $\mathcal{L} = \{(d_i, b_i)\}_{i=1}^N$. Here, $d_i$ denotes the natural language description of the $i$-th object, and $b_i = (x_i^0, y_i^0, x_i^1, y_i^1)$ specifies the bounding box coordinates, where $(x_i^0, y_i^0)$ and $(x_i^1, y_i^1)$ correspond to the top-left and bottom-right corners, respectively. We extract the object descriptions $\mathcal{D}=\{d_i\}_{i=1}^N$ from the prompt $p$ with a small language model (\cref{sec:label_set_generation}), and generate the layout conditioned on the object descriptions $\mathcal{D}$ and the prompt $p$ using a diffusion model (\cref{sec:architecture}). We provide an overview of our pipeline in \cref{fig:pipeline}.

\subsection{Description Set Generation} \label{sec:label_set_generation}

The description set $\mathcal{D}$ consists of natural language descriptions of objects of interest and their counts in the scene described by the prompt $p$.  To generate this set from $p$, we prompt a large language model (LLM) to extract noun phrases, assign a count to each noun phrase, and filter out the noun phrases that cannot be visualized in the scene. This allows us to generate the Description set $\mathcal{D}$ in an open-vocabulary manner, and our method can work in settings where object descriptions are not present or provided by the users. As demonstrated in \cref{tab:nsr_evaluation} and \cref{tab:effect_llm}, lightweight open-source LLMs such as Llama-3.1-8B \cite{dubey2024llama3herdmodels} are capable of generating description sets comparable in quality to those produced by closed-source and proprietary LLMs. We hypothesize that sentence parsing represents a fundamental task in language modeling, achievable by smaller-scale LLMs. The details for prompting are in the supplementary material, with representative examples shown \cref{tab:llm_label_set_predictions}.

\begin{table*}[!ht]
	\centering
	\resizebox{.8\linewidth}{!}{
	\begin{tabular}{ll}
		\toprule
		\textbf{Prompt}                                                   & \textbf{Description set}                 \\ \midrule
		There is a teapot and food on a plate.                            & teapot: 1, food: 1, plate: 1             \\
		Two men carrying plastic containers walking barefoot in the sand. & man: 2, plastic container: 2, sand: 1    \\
		A couple of children sitting down next to a laptop computer.      & child: 2, laptop: 1                      \\
		A man riding on the back of an elephant along a dirt road.        & man: 1, elephant: 1, dirt road: 1        \\
		Girl on a couch with her computer on a table                      & girl: 1, couch: 1, computer: 1, table: 1 \\
		\bottomrule
	\end{tabular}
	}
    \caption{Examples of description sets generated with LLama3.1-8B.}
	\label{tab:llm_label_set_predictions}
\end{table*}

\subsection{Conditional Layout Generation} \label{sec:architecture}

\subsubsection{Preliminaries}

We formulate the layout generation as a conditional generation problem that can be readily solved by conditional diffusion models (DDPM) \cite{ho2020denoisingdiffusionprobabilisticmodels}. Diffusion models are trained to generate high-quality samples from the target distribution by applying iterative denoising from pure Gaussian noise. We operate directly in the space of bounding box coordinates. Let $\mathcal{B}=\{b_i\}_{i=1}^N$ be the set of bounding boxes, where $b_i \in \mathbb{R}^4$ represents the bounding box coordinates of the $i^{th}$ object. In the forward diffusion process, given a starting sample from conditional layout distribution $\mathcal{B}_0 \sim p (\mathcal{B} | C)$, we generate a sequence of samples $\{\mathcal{B}_t\}_{t=1}^T$ by iteratively adding noise for $T$ timesteps:
\begin{equation}
	\mathcal{B}_t = \sqrt{\bar{\alpha}_t} \cdot \mathcal{B}_{0} + \sqrt{1 - \bar{\alpha}_t} \cdot \mathcal{E}_t
\end{equation}
where $\bar{\alpha}_t$ is the noise schedule, which decreases from 1 to 0 as $t$ goes from $0$ to $T$ in the diffusion process. $\mathcal{E}_t$ is a set with $N$ elements sampled from Gaussian noise $\mathcal{N}(0, I)$, and is element-wise added to $\mathcal{B}_0$. The condition $C$ is a combination of object description set $\mathcal{D}$, text prompt $p$, and aspect ratio $r=W / H$, where $W$ and $H$ are the width and height of the layout respectively.

A denoiser $\mathcal{E}_{\theta}$, parameterized by $\theta$, is trained to predict the noise added to the $\mathcal{B}_{0}$ at a given timestep $t$. The training objective is to minimize the MSE loss between added noise $\mathcal{E}_t$ and the predicted noise $\mathcal{E}_{\theta}(\mathcal{B}_{t},\,C,\,t)$:
\begin{equation}
	\mathcal{L}(\theta) = \mathbb{E}_{\mathcal{B}_0 \sim p(\mathcal{B} | C),\, t \sim \mathcal{U}(1, T)} \left[ \left\| \mathcal{E}_{\theta}(\mathcal{B}_{t},\,C,\,t) - \mathcal{E}_t \right\|^2 \right]
\end{equation}

We refer to the prompt $p$ as the global condition, the object description set $\mathcal{D}$ as the local condition, and the aspect ratio $r$ and timestep $t$ as scalar conditions.

\subsubsection{Layout Diffusion Transformer}
We model the denoiser $\mathcal{E}_{\theta}$ with a diffusion Transformer architecture~\cite{peebles2023scalable}, which applies a series of Transformer blocks on tokenized inputs to predict the added noise. We embed the bounding box coordinates $b_i$ into a fixed-size token of dimension $d_{model}$ with an MLP layer and encode the object description $d_i$ to a fixed-size token $d_{model}$ with T5-sentence encoder~\cite{ni2021sentencet5scalablesentenceencoders} followed by an MLP layer. We concatenate the encoded bounding box coordinates and object description embeddings to form the input token $t_i$ for the network, which is then passed through a series of modified DiT blocks.

We modify the standard DiT block to propagate global semantic information from the text prompt $p$ to local object description $d_i$, and bounding box coordinates $b_i$. We hypothesize that cross-attention between text-prompt and object description tokens can improve alignment between global and local conditions. We encode the text prompt $p$ using the Google-T5 model \cite{2020t5} and add a cross-attention layer to attend to object description and bounding box queries. Similar to bounding box tokens, object descriptions are passed through an MLP and modulated following the standard DiT block setting.

We normalize the bounding box coordinates by respective layout dimensions and rescale the coordinates to range $[-1, 1]$ prior to input to the denoiser for an aspect-ratio agnostic representation. We encode the global scalar conditions, aspect ratio $r$ and timestep $t$, with adaptive layer norm (adaLN) \cite{perez2017filmvisualreasoninggeneral}. Our final architecture is present in \cref{fig:pipeline}.

\subsubsection{Noise Schedule Scaling}
The signal-to-noise ratio significantly affects the performance of the diffusion model \cite{chen2023importancenoiseschedulingdiffusion}, and the low dimensionality of bounding box coordinates results in information being destroyed quickly in the initial phases of the denoising process, as demonstrated in the supplementary material. Previous works \cite{chen2023generalistframeworkpanopticsegmentation, chen2023diffusiondetdiffusionmodelobject} proposed to scale the input to the denoiser by a scaling factor $s$. However, this approach requires normalization of inputs for stable training \cite{chen2023importancenoiseschedulingdiffusion}. Instead, we integrate the scaling factor directly into the noise schedule $\alpha_t$:
\begin{equation} \label{eq:alpha}
	\bar{\alpha}'_t = \frac{\bar{\alpha}_t \cdot s^2}{1 + (\bar{\alpha}_t \cdot (s^2 - 1))}
\end{equation}
We visualize the effect of the scaling factor on the denoising process and provide complete proof of \cref{{eq:alpha}} in the supplementary material. Overall, $s > 1$ results in a more gradual destruction of information for the bounding box coordinates and improves the performance of the diffusion model, as demonstrated in our ablation study.

\section{Experiments}

We conduct a comprehensive set of experiments to demonstrate that \algoname{} achieves state-of-the-art performance on open-vocabulary scene layout generation and conduct ablation studies in \cref{sec:experiment_ablation_study} to demonstrate the effectiveness of our design choices pertaining to architecture, scaling factor, and description set generation.

\subsection{Datasets}

\paragraph{NSR-1K Spatial \cite{feng2023layoutgpt}} We use the NSR-1K spatial dataset proposed in LayoutGPT \cite{feng2023layoutgpt} to train our model for understanding the spatial relationship between objects present in the scene. The dataset contains 738 prompts describing four spatial relations: \textit{above}, \textit{below}, \textit{left}, and \textit{right} between two objects in the scene.

\paragraph{COCOGroundedDataset (COCO-GR)} COCO17 \cite{lin2015microsoftcococommonobjects} is a widely known dataset containing image-caption pairs along with object bounding boxes. However, there are two limitations for training \algoname{} on the COCO17 dataset: (1) The labels of bounding boxes are limited to 80 object classes, limiting the ability to train an open-vocabulary model, and (2) There is a low semantic overlap between the bounding boxes of objects in the image and the associated captions. To address these issues, we create a \textit{Grounded MS-COCO dataset} following \cite{peng2023kosmos2groundingmultimodallarge}, which we refer to as \textbf{COCO-GR}. We extract nouns present in the image captions with Llama-3.1-8B and obtain the bounding boxes for the extracted nouns using GroundingDINO. Our final dataset contains 578,951 layouts with an average of 5.62 objects per layout and an average prompt length of 9.91 words. We primarily use COCO-GR to compare the methods in a comprehensive manner with the FID \cite{heusel2018ganstrainedtimescaleupdate} score.

\paragraph{GRIT \cite{peng2023kosmos2groundingmultimodallarge}} GRIT is a large-scale dataset of Grounded Image-Text pairs created based on image-text pairs from COYO-700M and LAION-2B. We use this dataset for pretraining our model to efficiently learn the alignment between the global prompt and local object descriptions. During training, we utilize layouts that contain at least 3 objects, resulting in approximately 3.5 million layouts.

\subsection{Implementation Details} We use the Llama-3.1-8B model to generate the description set $\mathcal{D}$ from the text prompt $p$ and initialize a denoiser consisting of 8 layers, each with 8 attention heads for multi-head attention, and a hidden dimension of size 256, resulting in a model with $\sim$18M parameters. We use 100 diffusion steps at a noise schedule scale $s = 2.0$ for training and sample with 15 DDIM steps using timestep respacing \cite{nichol2021improved} during inference, unless otherwise stated. We use ADAM \cite{kingma2017adammethodstochasticoptimization} with a learning rate of $10^{-4}$ and a batch size of 256 on 2 NVIDIA RTX A5000 GPUs. We train two models: 1) LayouSyn, which is trained directly on the NSR-spatial and COCO-GR datasets for 800K steps, and 2) LayouSyn (GRIT pretraining), which is first pretrained on GRIT for 800K steps and then finetuned on the NSR-spatial and COCO-GR datasets with a learning rate of $1e^{-5}$ for 200K steps.

\begin{table}[!ht]
	\centering
	\begin{tabular}{lc}
		\toprule
		Model                          & L-FID $\downarrow$                             \\ \midrule
		LayoutGPT (GPT-3.5)            & 3.51                                           \\
		LayoutGPT (GPT-4o-mini)        & 6.72                                           \\
		Llama-3.1-8B (finetuned)       & 13.95                                          \\
		\midrule
		{\algoname}                    & \textbf{3.07} (\textcolor{darkgreen}{+12.5\%}) \\
		{\algoname} (GRIT pretraining) & 3.31 (\textcolor{darkgreen}{+5.6\%})           \\ \bottomrule
	\end{tabular}
	\caption{\textbf{Layout Quality Evaluation on the COCO-GR Dataset}: Our method outperforms existing layout generation methods on the FID (L-FID) score by at least 5.6\%.} \label{tab:combined_evaluation_coco}
\end{table}

\subsection{Layout FID Evaluation}\label{sec:experiment_open_vocabulary_evaluation_coco}

\begin{table*}[!ht]
	\centering

	\begin{tabular}{l aaac ac}
		\toprule
		                                                           & \multicolumn{4}{c}{\textbf{Numerical Reasoning}} & \multicolumn{2}{c}{\textbf{Spatial Reasoning}}                                                                                         \\
		\cmidrule(lr){2-5}\cmidrule(lr){6-7}
		\cmidrule(lr){2-3} \cmidrule(lr){4-5} \cmidrule(lr){6-6} \cmidrule(lr){7-7}
		                                                           & \makecell{Prec. $\uparrow$}                      & \makecell{Recall $\uparrow$}                   & Acc. $\uparrow$ & GLIP $\uparrow$   & \makecell{Acc. $\uparrow$ } & GLIP $\uparrow$   \\
		\midrule
		GT layouts                                                 & 100.0                                            & 100.0                                          & 100.0           & 50.08             & 100.00                      & 57.20             \\ \midrule

		\textbf{\underline{In-context Learning}}                   &                                                  &                                                &                 &                   &                             &                   \\
		LayoutGPT~(Llama-3.1-8B)                                   & 78.61                                            & 84.01                                          & 71.71           & 49.48             & 75.40                       & 47.92             \\
		LayoutGPT~(GPT-3.5)                                        & 76.29                                            & 86.64                                          & 76.72           & \underline{54.25} & 87.07                       & 56.89             \\
		LayoutGPT~(GPT-4o-mini) \cite{feng2023layoutgpt}           & 73.82                                            & 86.84                                          & 77.51           & \underline{57.96} & 92.01                       & \underline{60.49} \\ \midrule
		\textbf{\underline{Zero-shot}}                             &                                                  &                                                &                 &                   &                             &                   \\
		LLMGroundedDiffusion (GPT-4o-mini) \cite{lian2023llm}      & 84.36                                            & 95.94                                          & 89.94           & 38.56             & 72.46                       & 27.09             \\
		LLM Blueprint (GPT-4o-mini)  \cite{gani2023llm}            & \textbf{87.21}                                   & 67.29                                          & 38.36           & 42.24             & 73.52                       & 50.21             \\
		\midrule
		\textbf{\underline{Trained /Finetuned}}                    &                                                  &                                                &                 &                   &                             &                   \\
		LayoutTransformer \cite{gupta2021layouttransformer} *      & 75.70                                            & 61.69                                          & 22.26           & 40.55             & 6.36                        & 28.13             \\
		Ranni \cite{feng2023ranni}                                 & 56.23                                            & 83.28                                          & 40.80           & 38.19             & 53.29                       & 24.38             \\
		Llama-3.1-8B (finetuned)  \cite{dubey2024llama3herdmodels} & 79.33                                            & 93.36                                          & 70.84           & 44.72             & 86.64                       & 52.93             \\
		\midrule
		\textbf{\underline{Ours}}                                  &                                                  &                                                &                 &                   &                             &                   \\
		{\algoname{}}                                              & 77.62                                            & 99.23                                          & 95.14           & \underline{56.17} & 87.49                       & 54.91             \\
		{\algoname{} (GRIT pretraining)}                           & 77.62                                            & \textbf{99.23}                                 & \textbf{95.14}  & \underline{56.20} & \textbf{92.58}              & \underline{58.94} \\
		\bottomrule
	\end{tabular}

	\caption{\textbf{Spatial and Counting Evaluation on the NSR-1K Benchmark}. \algoname{} outperforms existing methods on spatial and counting reasoning tasks, achieving state-of-the-art performance on most metrics. Note: $*$ indicates metrics reported by LayoutGPT~\cite{feng2023layoutgpt}, and \textcolor{gray}{\textbf{shaded columns}} represent metrics computed directly on the generated layouts. We \textbf{bold} values for metrics where the ground truth (GT) is 100\% and \underline{underline} values where methods exceed the ground truth performance.}
	\label{tab:nsr_evaluation}
\end{table*}

\begin{figure}[!ht]
	\centering
	\includegraphics[width=\linewidth]{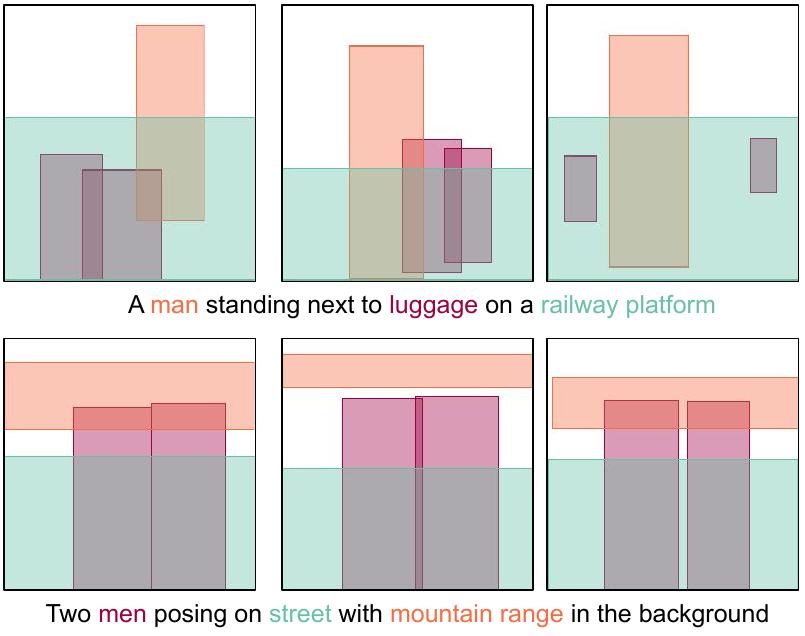}
	\caption{Diversity of layouts generated by \algoname{}}
	\label{fig:open_vocabulary_examples_coco}
\end{figure}

We evaluate the quality of layouts generated by \algoname{} with the FID~\cite{heusel2018ganstrainedtimescaleupdate} metric on the COCO-GR validation dataset. For a fair evaluation, we limit our comparison to methods that can directly or indirectly utilize COCO-GR training datasets and use two baselines: 1) LayoutGPT with in-context prompting and 2) Llama-3.1-8B finetuned on COCO-GR and NSR-spatial dataset. We add the COCO-GR training dataset to the in-context exemplars used by LayoutGPT. For Llama finetuning, we use LoRA~\cite{hu2022lora} with $r=8$ and fine-tune for 10K steps.

For FID calculation, we draw the layout as an image and map each object to a specific color following document layout generation literature \cite{wang2024dolfin}, taking into account semantic similarity between different objects based on CLIP \cite{radford2021learningtransferablevisualmodels} similarity. Due to cost constraints with GPT models, we limit our evaluation to the first 8,700 captions from the COCO-GR validation dataset and use LayoutGPT with GPT-3.5 and GPT-4o-mini. The results are shown in \cref{tab:combined_evaluation_coco}. We outperform LayoutGPT on Layout FID metric by a significant 12.50\% and 5.6\%, without and with GRIT pretraining respectively, demonstrating the superiority of our proposed method in open-vocabulary layout generation. Moreover, we qualitatively demonstrate the diversity of generated layouts in \cref{fig:open_vocabulary_examples_coco}, and show comparisons with LayoutGPT in \cref{fig:qual_compare_layoutgpt}.

\subsection{Spatial and Numerical Evaluation}
We evaluate \algoname{} on the NSR-1K spatial and numerical reasoning benchmark and compare extensively with state-of-the-art methods, including LayoutGPT \cite{feng2023layoutgpt}, LLMGroundedDiffusion \cite{lian2023llm}, LLM Blueprint \cite{gani2023llm}, Ranni \cite{feng2023ranni}, and Llama-3.1-8B finetuned on the COCO-GR and NSR-Spatial dataset. We use GLIGEN~\cite{li2023gligenopensetgroundedtexttoimage} to generate images from layouts and re-run GLIGEN on the layouts reported in LayoutGPT due to a lack of original hyperparameters. We briefly describe the metrics below for completeness and refer the readers to LayoutGPT \cite{feng2023layoutgpt} for more details.

\paragraph{Numerical Reasoning} We evaluate the numerical quality of the generated layouts on Precision, Recall, Accuracy, and GLIP accuracy. \textit{Precision} is the percentage of predicted objects in the ground-truth objects set, and \textit{Recall} is the percentage of ground-truth objects in the predicted object set. \textit{Accuracy} for a test example is defined as 1 if the ground-truth object set and predicted object set overlap exactly and 0 otherwise. The \textit{GLIP accuracy} for a test example is defined as 1 if the GLIP detected object count matches the ground-truth object count and 0 otherwise.

\paragraph{Spatial Reasoning} We evaluate spatial reasoning on accuracy and GLIP accuracy. \textit{Accuracy} and \textit{GLIP accuracy} for a test example are defined as 1 if the predicted object locations in the layout and GLIP detected bounding box follow the spatial constraints, and 0 otherwise.

The results on the NSR-1K benchmark are reported in \cref{tab:nsr_evaluation}. \algoname{} achieves superior performance across multiple metrics, including 92.58\% accuracy in spatial reasoning, 58.94\% GLIP detection accuracy, which is higher than ground truth layouts, and a high recall and accuracy of 99.23\% and 95.14\% respectively on the numerical reasoning benchmark. Note that a high recall indicates a very high overlap between the predicted and ground-truth object sets, indicating that smaller language models can be effectively used for object label generation.

\begin{figure}[!ht]
	\centering
	\includegraphics[width=\linewidth,trim=0.5em 0 0.5em 0,clip]{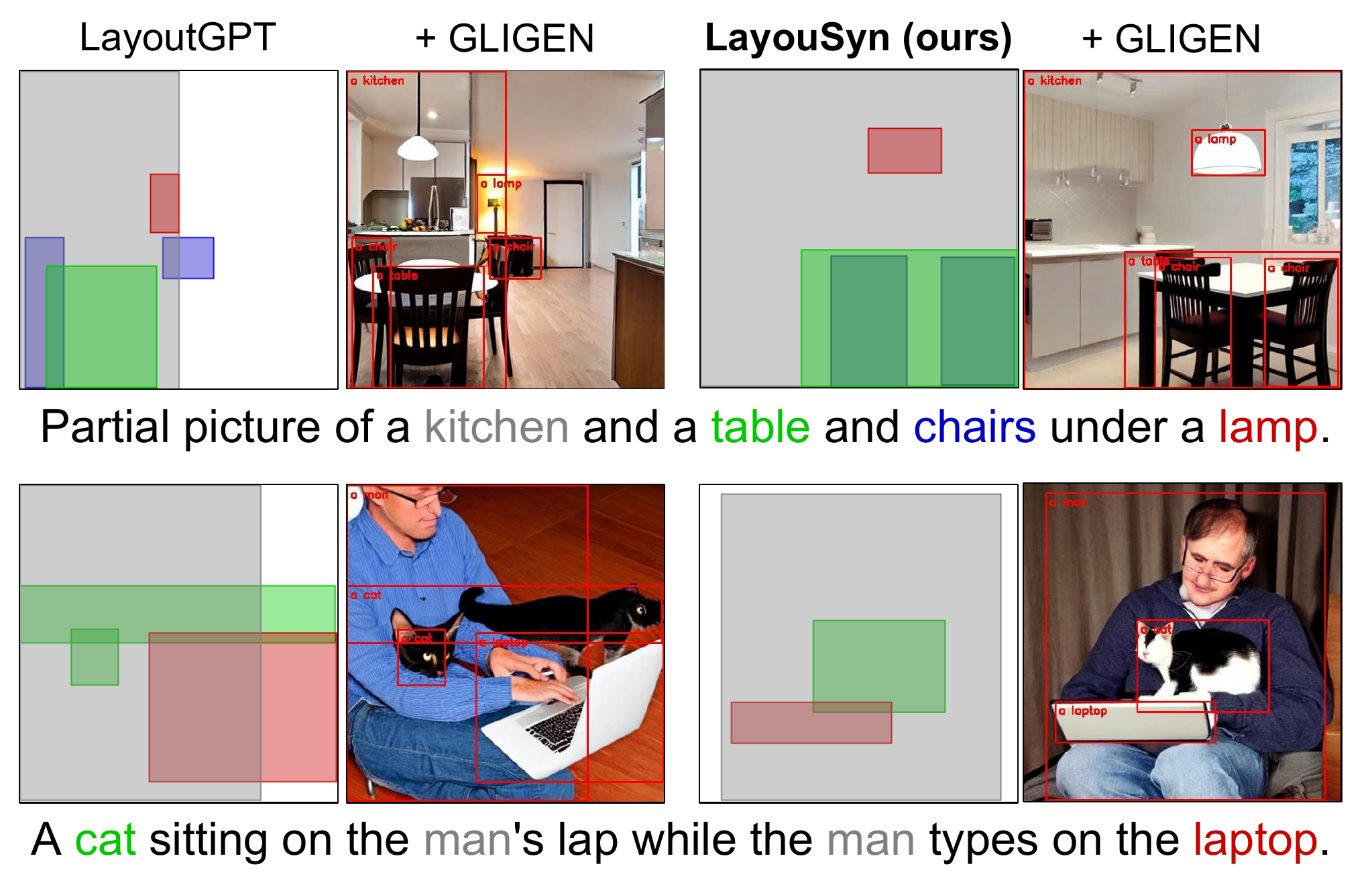}
	\caption{{\bf Comparative analysis with LayoutGPT}. \algoname{} can generate complex layouts with multiple objects following spatial constraints in the prompt.}
    \label{fig:qual_compare_layoutgpt}
\end{figure}

\subsection{Ablation Study} \label{sec:experiment_ablation_study}
\begin{figure}[!ht]
        \centering
        \includegraphics[width=\linewidth]{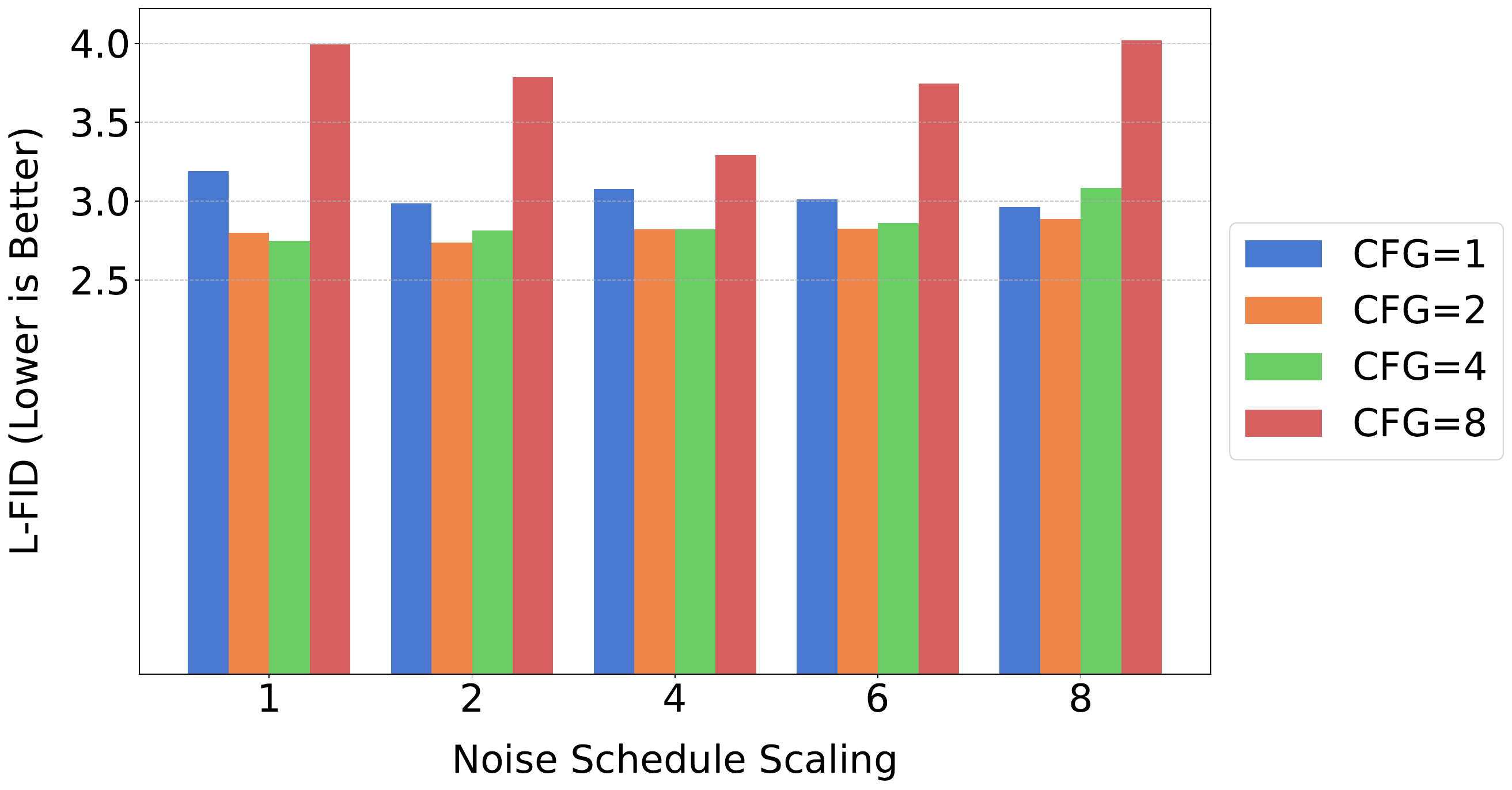}
        \caption{\textbf{Effect of noise schedule scaling on layout fidelity (L-FID) at different CFG scales}. A noise schedule scale of 2 and a CFG scale of 2 achieves the lowest L-FID.}
        \label{fig:effect_scale}
\end{figure}

We use the LDiT model trained on the COCO-GR and NSR-spatial datasets without pretraining and fix the number of sampling steps to match the diffusion steps (i.e., 100) to eliminate any potential noise from the sampling process, unless otherwise stated.

\paragraph{Noise schedule scaling} \label{sec:experiment_ablation_scale}

We report the quantitative results for different noise schedule scales $\bar{\alpha}'_t$ (\cref{eq:alpha}) and classifier-free guidance (CFG)~\cite{ho2022classifier} scales in \cref{fig:effect_scale}. We observe that L-FID initially decreases with increasing input scale but then begins to rise. We attribute the performance decline at higher scales to the noise schedule dropping too quickly during the later diffusion steps (visualized in the supplementary material). Overall, the model trained with an input scale of 2.0 and a CFG scale of 2.0 achieves the lowest L-FID value.

\begin{table}[!ht]
	\centering
	\begin{tabular}{lccc} \toprule
		Desc Set ($\rightarrow$) & GPT-3.5                               & GPT-4o-mini                           \\ \midrule
		LayoutGPT                & 3.51                                  & 6.72                                  \\
		Layousyn                 & 2.86 (\textcolor{darkgreen}{+18.5\%}) & 3.32 (\textcolor{darkgreen}{+50.1\%}) \\ \bottomrule
	\end{tabular}
	\caption{\textbf{Layout generation with a fixed description set }. We show layout FID (L-FID) scores for layout quality evaluation.}
	\label{tab:fixed_label_set}
\end{table}

\paragraph{Layout Generation with a Fixed Description Set} To isolate the layout generation capability of \algoname{} from the effects of LLM-based description set generation (\cref{sec:label_set_generation}), we evaluate \algoname{} using the same description set generated by LayoutGPT in \cref{sec:experiment_open_vocabulary_evaluation_coco}. As shown in \cref{tab:fixed_label_set}, \algoname{} achieves up to a 50.1\% improvement in Layout-FID (L-FID) compared to LayoutGPT, demonstrating superior layout generation performance.

\begin{table}[!ht]
    \centering
    \begin{tabular}{lccc}
        \toprule
        & GPT-3.5 & GPT-4o-mini & Llama-3.1-8B \\ \midrule
        L-FID $\downarrow$ & 3.49 & 3.22 & \textbf{2.74} \\ \bottomrule
    \end{tabular}
    \caption{\textbf{L-FID on description sets generated by different LLMs}. Llama-3.1-8B can effectively extract description set.} \label{tab:effect_llm}
\end{table}

\paragraph{Description Sets from Different LLMs} \label{sec:experiment_ablation_object_label_generation}

We evaluate \algoname{} using description sets extracted from LLama3.1-8B-Instruct, GPT-3.5-chat, and GPT-4o-mini. The results are shown in \cref{tab:effect_llm}. The findings confirm that lightweight language models effectively generate object descriptions, as extracting object labels from prompts is a simpler task than full layout generation.

\begin{table}[!ht]
	\centering
	\begin{tabular}{lcccccc}
		\toprule
		DDIM-Sampling \#   & 10   & 15   & 20   & 50   & 100  \\
		\midrule
		L-FID $\downarrow$ & 3.47 & 3.07 & 2.94 & 2.81 & 2.74 \\
		\bottomrule
	\end{tabular}
    \caption{\textbf{Effects of \# DDIM sampling steps on L-FID}. A step size of 15 can generate layouts with high fidelity.}\label{tab:sampling_steps}
\end{table}

\paragraph{Number of sampling steps} In \cref{tab:sampling_steps}, we demonstrate the trade-offs between generation quality and inference speed. On a single RTX A6000 card, we measured $\sim$2 sampling steps per second with $\sim$1400 samples, translating to $\sim$ 5 ms to generate a sample with 15 steps. The time cost for inference scales linearly with the number of sampling steps required. \algoname{} can achieve comparable results to LayoutGPT with 10 steps and outperforms LayoutGPT with only 25 steps ($\sim$0.2 seconds).

\begin{table}[!ht]
	\centering
	
	\begin{tabular}{ccc}
		\toprule
		Cross-Attention & Modulation & L-FID $\downarrow$ \\
		\midrule
		\ding{55}       & \ding{55}  & 2.82               \\
		\checkmark      & \ding{55}  & 2.81               \\
		\checkmark      & \checkmark & \textbf{2.74}      \\
		\bottomrule
	\end{tabular}
    \caption{Ablation study on architectural components of LDiT block, measuring their impact on L-FID score. Alignment between description tokens and the global prompt with cross-attention and modulation leads to a lower L-FID score.}
	\label{tab:ablation_architecture}
\end{table}

\paragraph{LDiT Architecture} We study the performance improvements of LDiT from evolving description tokens by comparing with 1) Vanilla DiT architecture without cross attention and modulation of description tokens, 2) Vanilla DiT architecture with added cross attention between description tokens and global prompt but without modulation of description tokens. The results are shown in \cref{tab:ablation_architecture}. We observe that evolving description tokens leads to lower FID and improved results. We quantitatively visualize the results from different models in \cref{tab:ablation_architecture}.

\section{Applications}

\subsection{LLM Initialization}

\algoname{} can be integrated with an LLM, using its planned layouts as initialization and refining them to achieve better performance with equal or fewer sampling steps. Specifically, we take the outputs from LayoutGPT, which can be used with different LLMs. For initialization, we design two strategies: 1) \textit{Description set only}: use only the description sets $\mathcal{D}$ predicted by the LLM and perform denoising starting from Gaussian noise. Full 100 denoising steps are executed; 2) \textit{Description Set + Inversion}: in addition to using the description sets, apply DDIM inversion~\cite{couairon2022diffedit} on the bounding boxes predicted by the LLM. We only denoise for the same number of steps as inversion.

We present spatial reasoning evaluations in \cref{tab:llm_init}. When using only description sets with inversion, \algoname{} brings improvement in accuracy for Llama-3.1-8B ($+$15.06), GPT-3.5 ($+$5.3), and GPT-4o-mini ($+$0.07). Comparing the results from Gaussian noise initialization (Description Set) with those from DDIM inversion (Description Set + Inv), the latter consistently yields higher accuracy, requiring just 15 steps, regardless of the LLM used. This highlights the effectiveness of LLM initialization compared to pure Gaussian noise, even when the LLM predictions are coarse.

\begin{table}[!ht]
	\centering
	\label{tab:llm_init}
    
	\resizebox{\linewidth}{!}{
		\begin{tabular}{l c c c}
			\toprule
			Method                     & Llama-3.1-8B   & GPT-3.5        & GPT-4o-mini    \\
			\midrule
			Original                   & 75.40          & 87.07          & 92.01          \\
			Description Set            & 89.75          & 90.04          & 90.95          \\
			Description Set + Inv (15) & \textbf{90.46} & \textbf{92.37} & \textbf{92.08} \\
			\bottomrule
		\end{tabular}
	}
	\caption{Spatial reasoning results with LLM initialization. Description Set: using the description set derived from LLM; Inv: initialize bounding boxes with DDIM inversion of LLM predictions (numbers in brackets are steps of inversion performed).}
\end{table}

\subsection{Object Addition}
\label{sec:addition}

\begin{figure}[!ht]
	\centering
	\includegraphics[width=\linewidth]{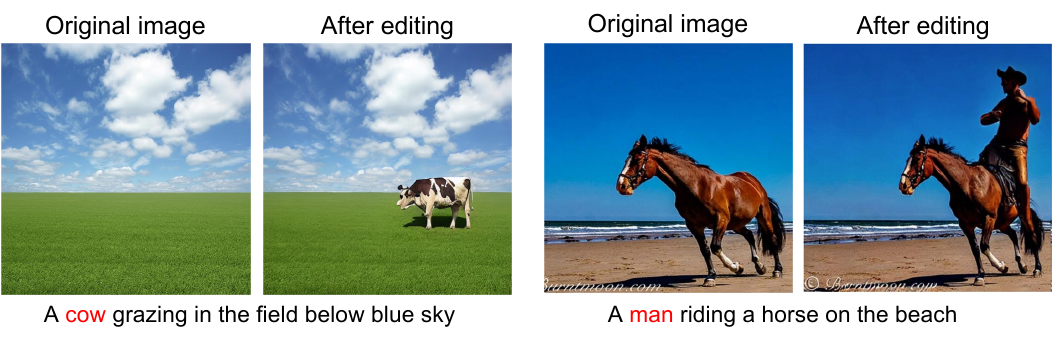}
	\caption{\textbf{Examples of automated object addition using \algoname{}} Objects to add are highlighted in red in the prompt. Our pipeline consists of four steps: extracting relevant objects from the prompt with LLM, detecting objects present in the scene with GroundingDINO \cite{liu2023grounding}, layout completion of the object to add with \algoname{} (ours), and finally inpainting the object in the image with GLIGEN \cite{li2023gligenopensetgroundedtexttoimage}.}
	\label{fig:sample_obj_addition}
\end{figure}

Image inpainting \cite{lugmayr2022repaintinpaintingusingdenoising} with the diffusion model is widely used for adding objects to images. However, these models need users to specify the spatial location of the objects to be added, requiring a human-in-the-loop to guide the inpainting process. We demonstrate the capability of automatic object addition without human intervention with \algoname{}. We outline components of our pipeline below and visualize examples in \cref{fig:sample_obj_addition}.

\begin{enumerate}
	\item \textbf{Description set:} We use an LLM to generate an object set $\mathcal{D}$ from the prompt $p$. $\mathcal{D}$ contains a list of objects that need to be considered during the object addition process.
	\item \textbf{Object Detection:} We use a pre-trained object detection model to detect objects from the description set $\mathcal{D}$ in the image $I$. We obtain a set of bounding boxes $\mathcal{B}$ for the detected objects and create a layout $L$ with the obtained bounding boxes $\mathcal{B}$ and description set $\mathcal{D}$.
	\item \textbf{Layout Completion:} We inpaint \cite{lugmayr2022repaintinpaintingusingdenoising} the bounding box locations of the objects to add with \algoname{}, and obtain an inpainting mask $M$ based on the predicted bounding boxes for objects in $\mathcal{A}$.
	\item \textbf{Object Inpainting:} We use inpainting \cite{lugmayr2022repaintinpaintingusingdenoising} with GLIGEN \cite{li2023gligenopensetgroundedtexttoimage} to inpaint objects in the set $\mathcal{A}$ into the image $I$ using the inpainting mask $M$.
\end{enumerate}

\section{Conclusion} \label{sec:conclusion}

We present \textbf{\algoname{}}, a novel open-vocabulary scene-aware text-to-layout generation framework. \textbf{\algoname{}} adopts lightweight open-source language models to predict objects from text prompts and a novel open-vocabulary diffusion-Transformer-based architecture trained in a scene-aware manner to generate layouts at any aspect ratio. Our method achieves state-of-the-art results across multiple benchmarks, including complex spatial and numerical reasoning tasks. Further, we demonstrate an interesting finding: we can seamlessly combine initialization from LLMs to reduce diffusion sampling steps and refine the LLM predictions. Finally, we present a pipeline for adding objects to the image, demonstrating the potential of \algoname{} in image editing applications. In the future, we plan to incorporate additional geometrical constraints, such as depth maps, to better handle occlusion scenarios.

\paragraph{Acknowledgment} We would like to thank Haiyang Xu, Ethan Armand, and Zeyuan Chen for their valuable feedback and insightful discussions. This work is supported by NSF Award IIS-2127544 and NSF Award IIS-2433768.

{
	\small
	\bibliographystyle{ieeenat_fullname}
	\bibliography{main}
}

\appendix
\counterwithin{figure}{section}
\counterwithin{table}{section}
\setcounter{table}{0}
\setcounter{figure}{0}
\onecolumn
\section{Appendix}
\subsection{Scaling Factor}\label{sec:appendix_scaling_factor}

\begin{figure}[!ht]
	\centering
	\includegraphics[width=0.5\linewidth]{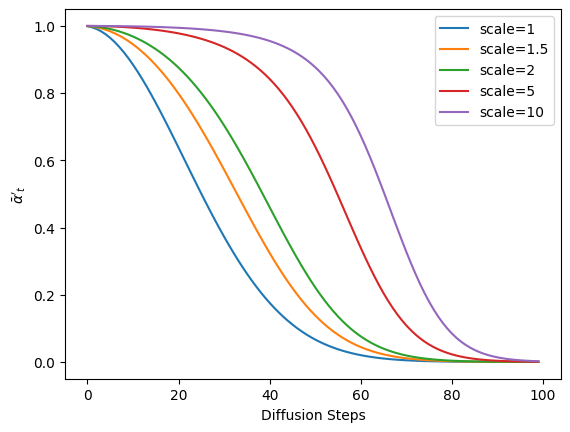}
	\caption{\textbf{Effect of scaling factor on denoising Process}. We plot the noise schedule $\bar{\alpha}'_t$ for diffusion process with 100 steps for different scaling factors $s$. We observe that $s > 1$ results in a more gradual destruction of information.}
\end{figure}

\begin{figure}[t]
    \centering
    \captionsetup[subfigure]{labelformat=empty}
    \setlength{\fboxsep}{0pt} 
    \begin{subfigure}[b]{0.32\linewidth}
        \centering
        \fbox{\includegraphics[width=\linewidth]{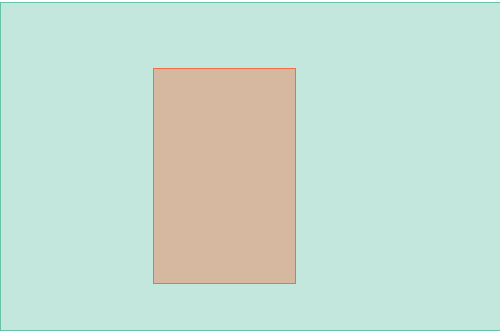}}
        \caption{Step 0}
    \end{subfigure}
    \hfill
    \begin{subfigure}[b]{0.32\linewidth}
        \centering
        \fbox{\includegraphics[width=\linewidth]{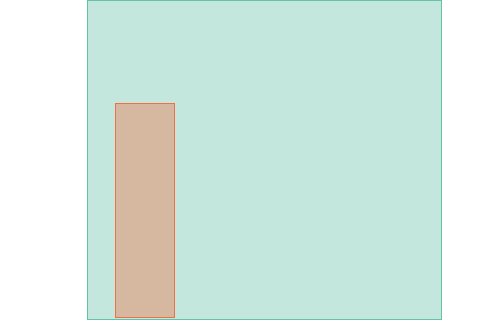}}
        \caption{Step 5}
    \end{subfigure}
    \hfill
    \begin{subfigure}[b]{0.32\linewidth}
        \centering
        \fbox{\includegraphics[width=\linewidth]{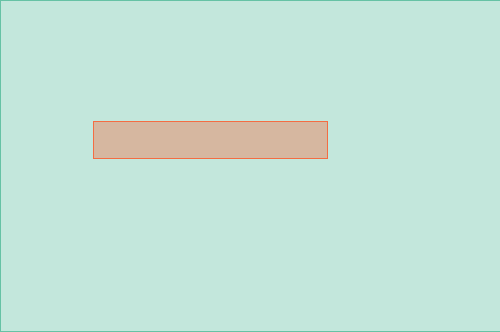}}
        \caption{Step 10}
    \end{subfigure}
    \hfill
    \begin{minipage}[b]{\linewidth}
        \centering
        \caption*{\textbf{Scale = 1.0}}
    \end{minipage}
    \vskip\baselineskip
    \begin{subfigure}[b]{0.32\linewidth}
        \centering
        \fbox{\includegraphics[width=\linewidth]{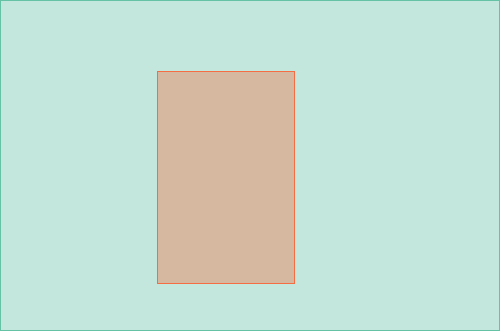}}
        \caption{Step 0}
    \end{subfigure}
    \hfill
    \begin{subfigure}[b]{0.32\linewidth}
        \centering
        \fbox{\includegraphics[width=\linewidth]{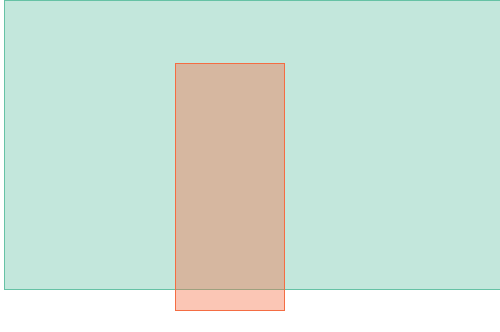}}
        \caption{Step 5}
    \end{subfigure}
    \hfill
    \begin{subfigure}[b]{0.32\linewidth}
        \centering
        \fbox{\includegraphics[width=\linewidth]{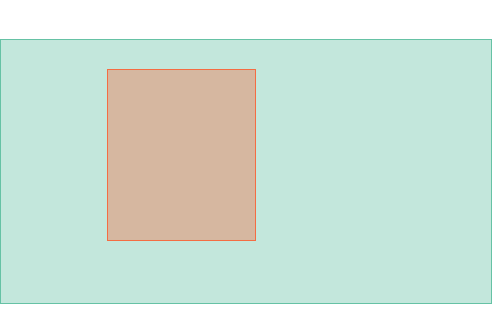}}
        \caption{Step 10}
    \end{subfigure}
    \begin{minipage}[b]{\linewidth}
        \centering
        \caption*{\textbf{Scale = 2.0}}
    \end{minipage}
    \vskip\baselineskip 
    \begin{subfigure}[b]{0.32\linewidth}
        \centering
        \fbox{\includegraphics[width=\linewidth]{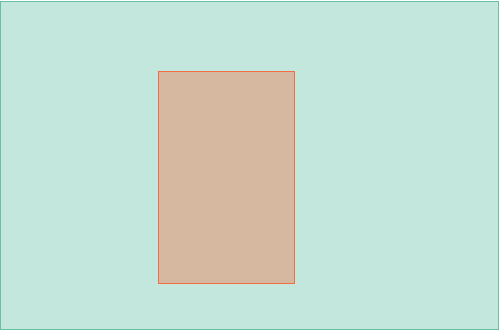}}
        \caption{Step 0}
    \end{subfigure}
    \hfill
    \begin{subfigure}[b]{0.32\linewidth}
        \centering
        \fbox{\includegraphics[width=\linewidth]{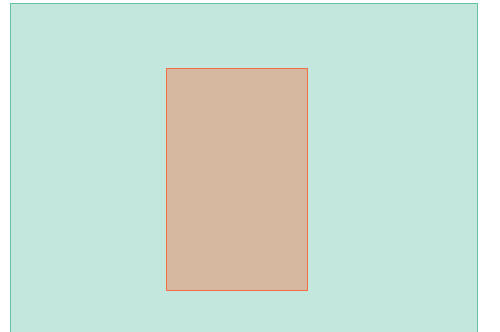}}
        \caption{Step 5}
    \end{subfigure}
    \hfill
    \begin{subfigure}[b]{0.32\linewidth}
        \centering
        \fbox{\includegraphics[width=\linewidth]{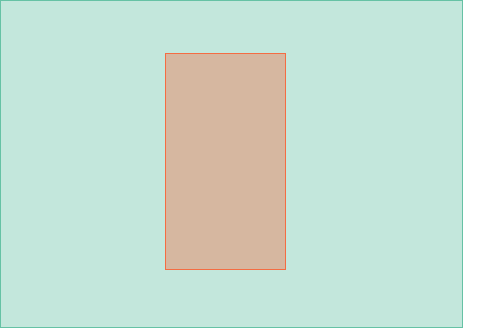}}
        \caption{Step 10}
    \end{subfigure}
    \begin{minipage}[b]{\linewidth}
        \centering
        \caption*{\textbf{Scale = 5.0}}
    \end{minipage}
    \caption{\textbf{Visualizing denoising process scale 1.0, 2.0, and 5.0}. The denoising process for higher scaling factor results in a more gradual destruction of information for the bounding box coordinates for layout with prompt \textit{\textcolor[RGB]{244, 109, 67}{Snowboarder} cuts his way down a \textcolor[RGB]{102, 194, 165}{ski slope}}.}
    \label{fig:scaling_factor}
\end{figure}

\newtheorem{theorem}{Theorem}

\begin{theorem} \label{theorem:scaling_factor}
	Given the forward process scaled by a factor $s$ and normalized input distribution
	\begin{equation}
		X_t = s \sqrt{\alpha_t} X_0 + \sqrt{1 - \alpha_t} \epsilon_t, \quad \epsilon_t \sim \mathcal{N}(0, 1), \quad \mathbb{E}[X_0] = 0, \quad \text{Var}(X_0) = 1.
	\end{equation}
	the normalized process $\tilde{X}_t$ is given by
	\begin{equation}
		\tilde{X}_t = \sqrt{\tilde{\alpha}_t} X_0 + \sqrt{1 - \tilde{\alpha}_t} \epsilon_t,
	\end{equation}
	where
	\begin{equation}
		\tilde{\alpha}_t = \frac{\sqrt{\alpha_t} s}{\sqrt{(s^2 - 1)\alpha_t + 1}}
	\end{equation}
	and has the property that $E[\tilde{X}_t] = 0$ and $Var(\tilde{X}_t) = 1$.	
\end{theorem}

\begin{proof}
	We start with the expression for $X_t$:
	\begin{equation}
		X_t = s \sqrt{\alpha_t} X_0 + \sqrt{1-\alpha_t} \epsilon_t
	\end{equation}
	\textbf{Step 1: Expectation of $X_t$}\\
	Taking the expectation of both sides:
	\begin{equation}
		\mathbb{E}[X_t] = \mathbb{E}\left[ s \sqrt{\alpha_t} X_0 + \sqrt{1-\alpha_t} \epsilon_t \right]
	\end{equation}
	Since $\mathbb{E}[X_0] = 0$ and $\mathbb{E}[\epsilon_t] = 0$, it follows that:
	\begin{equation}
		\mathbb{E}[X_t] = s \sqrt{\alpha_t} \cdot \mathbb{E}[X_0] + \sqrt{1-\alpha_t} \cdot \mathbb{E}[\epsilon_t] = 0.
	\end{equation}
	Thus,
	\begin{equation}
		\mathbb{E}[X_t] = 0.
	\end{equation}
	\textbf{Step 2: Variance of $X_t$}\\
	Next, we compute the variance of $X_t$:
	\begin{equation}
		\text{Var}(X_t) = s^2 \alpha_t \text{Var}(X_0) + (1-\alpha_t) \text{Var}(\epsilon_t).
	\end{equation}
	Since $\text{Var}(X_0) = 1$ and $\text{Var}(\epsilon_t) = 1$, we have:
	\begin{equation}
		\text{Var}(X_t) = s^2 \alpha_t + (1-\alpha_t) = \alpha_t (s^2 - 1) + 1
	\end{equation}
	\textbf{Step 3: Normalization of $X_t $}\\
	We define the normalized process $\tilde{X}_t $ as:
	\begin{equation}
		\tilde{X}_t = \frac{X_t - \mathbb{E}[X_t]}{\sqrt{\text{Var}(X_t)}} = \frac{s \sqrt{\alpha_t} X_0 + \sqrt{1-\alpha_t} \epsilon_t}{\sqrt{\alpha_t (s^2 - 1) + 1}}
	\end{equation}
	This can be simplified to:
	\begin{equation}
		\tilde{X}_t = \sqrt{\tilde{\alpha}_t} X_0 + \sqrt{1 - \tilde{\alpha}_t} \epsilon_t
	\end{equation}
	where
	\begin{equation}
		\tilde{\alpha}_t = \frac{\sqrt{\alpha_t} s}{\sqrt{(s^2-1)\alpha_t + 1}}
	\end{equation}
	This completes the proof.
\end{proof}

\subsection{Description set generation}\label{sec:appendix_label_set_generation}

We use a pre-trained LLM to generate a set of object descriptions $\mathcal{D}$ from a given prompt.  The LLM is instructed to follow a structured process to extract noun phrases from the prompt that can be depicted in the scene and to output these objects along with their counts in JSON format. The prompt we use is as follows:

\lstset{basicstyle=\footnotesize\ttfamily,breaklines=true}

\begin{lstlisting}[breaklines]
    You are a creative scene designer who predicts a scene from a natural language prompt. A scene is a JSON object containing a list of noun phrases with their counts {"phrase1": count1, "phrase2": count2, ...}. The noun phrases contain **ONLY** common nouns. You strictly follow the below process for predicting plausible scenes: 

    Step 1: Extract noun phrases from the prompt. For example, "happy people", "car engine", "brown dog", "parking lot", etc.
    Step 2: Limit noun phrases to common nouns and convert the noun phrase to its singular form. For example, "happy people" to "person", "tall women" to "woman", "group of old people" to "person", "children" to "child", "brown dog" to "dog", "parking lot" remains "parking lot", etc.
    Step 3: Predict the count of each noun phrase and ensure consistency with the count of other objects in the scene. If a particular object does not have any explicit count mentioned in the prompt, use your creativity to assign a count to make the overall scene plausible but not too cluttered. For example, if the prompt is "a group of young kids playing with their dogs," the count of "kid" can be 3, and the count of "dog" should be the same as the count of "kid".
    Step 4: Output the final scene as a JSON object, only including physical objects and phrases without referring to actions or activities.

    Complete example:

    Prompt: Three white sheep and few women walking down a town road.
    Steps:
    Step 1: noun phrases: white sheep, women, town road
    Step 2: noun phrase in singular form: sheep, woman, town road
    Step 3: Since the count of women is not mentioned, we will assign a count of 2 to make the scene plausible. The count of "sheep" is 3 and the count of "town road" is 1.
    Step 4: {"sheep": 3, "woman": 2, "town road": 1}
    Plausible scene: {"sheep": 3, "woman": 2, "town road": 1}

    Other examples with skipped step-by-step process: 

    Prompt: A desk and office chair in the cubicle 
    Plausible scene: {"office desk": 1, "office chair": 1, "cubicle": 1} 

    Prompt: A pizza is in a box on a corner desk table.
    Plausible scene: {"pizza": 1, "box": 1, "desk table": 1}

    Note: Print **ONLY** the final scene as a JSON object.
\end{lstlisting}

\begin{figure}[!ht]
    \centering
    \includegraphics[width=\textwidth]{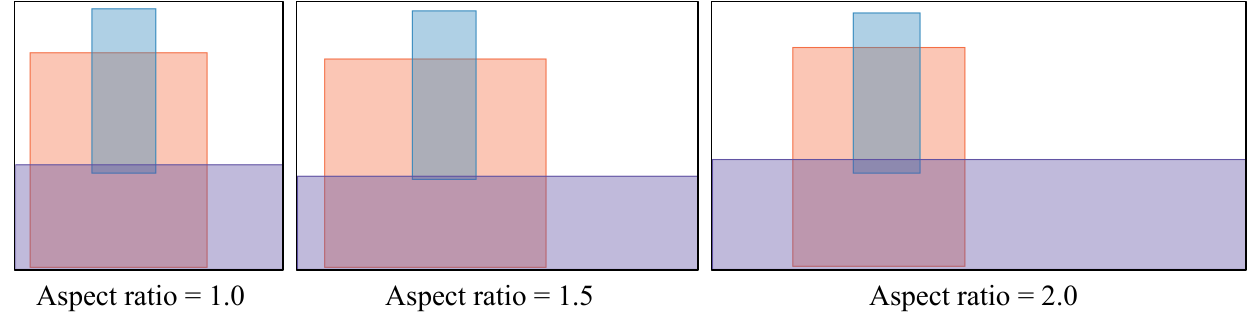}
    \vspace{-2em}
    \caption{\textbf{Layout generation with varying aspect ratios}. Layouts generated at different aspect ratios for prompt: \textit{A \textcolor[RGB]{50, 136, 189}{man} riding a \textcolor[RGB]{244, 109, 67}{horse} on the \textcolor[RGB]{94, 79, 162}{street}.} The model adjusts the position and aspect ratio corresponding to the \textcolor[RGB]{50, 136, 189}{man} and the \textcolor[RGB]{244, 109, 67}{horse} to produce natural-looking layouts.
    }
    \label{fig:aspect_ratio_examples}
\end{figure}

\subsection{Aspect Ratio}

Our model can generate layouts at different aspect ratios by conditioning on the aspect ratio of the image. We visualize the layouts generated at different aspect ratios in \cref{fig:aspect_ratio_examples} for a given caption. We observe that our model can generate layouts consistent with the image's aspect ratio.

\begin{figure}[!ht]
    \centering
    \captionsetup[subfigure]{labelformat=empty}
    \begin{subfigure}[b]{0.3\linewidth}
        \centering
        \fbox{\includegraphics[width=.9\linewidth]{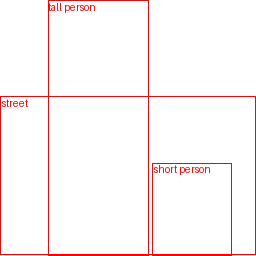}}
        \caption{\textbf{Prompt}: Two people on the street \\ \textbf{Description set}: tall person, short person, street}
    \end{subfigure}
    \hspace{3em}
    \begin{subfigure}[b]{0.3\linewidth}
        \centering
        \fbox{\includegraphics[width=.9\linewidth]{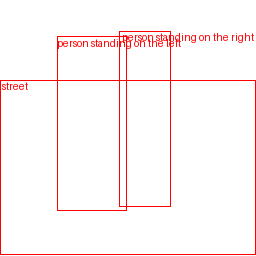}}
        \caption{\textbf{Prompt}: Two people on the street \\ \textbf{Description set}: person standing on the left, person standing on the right, street}
    \end{subfigure}
    \caption{\textbf{Visualizing generalization of our model on fine-grained description set}}
    \label{fig:label_set_generalization}
\end{figure}

\subsection{Description Set Generalizability} 

The COCO-GR dataset provides labels that are restricted to object names (e.g., “person”) and does not include detailed descriptions paired with these object names (e.g., “tall person”). However, training with the description set in an open-domain manner allows our model to generalize to fine-grained object descriptions. The results are visualized in \cref{fig:label_set_generalization}.

\end{document}